\documentclass{article}

    \PassOptionsToPackage{numbers, compress}{natbib}

\usepackage[preprint]{neurips_2025}

\usepackage[utf8]{inputenc} 
\usepackage{hyperref}       
\usepackage{url}            
\usepackage{booktabs}       
\usepackage{amsfonts}       
\usepackage{nicefrac}       
\usepackage{microtype}      
\usepackage{xcolor}         
\usepackage{amsmath}
\usepackage{graphicx}
\usepackage{enumitem}
\usepackage{algorithm}
\usepackage{algorithmic}
\usepackage{amsthm}

\newtheorem{theorem}{Theorem}
\newtheorem{proposition}{Proposition}

\title{From Observation to Orientation: an Adaptive Integer Programming Approach to Intervention Design}

\author{%
  Abdelmonem Elrefaey
    \\
  School of Computing and Augmented Intelligence\\
  Arizona State University\\
  Tempe, AZ 85281 \\
  \texttt{aelrefae@asu.edu} \\
  \And
  Rong Pan
    \\
  School of Computing and Augmented Intelligence\\
  Arizona State University\\
  Tempe, AZ 85281 \\
  \texttt{rpan1@asu.edu} \\
}

\begin{document}

\maketitle

\begin{abstract}
  Using both observational and experimental data, a causal discovery process can identify the causal relationships between variables. A unique adaptive intervention design paradigm is presented in this work, with the objective of learning causal directed acyclic graphs (DAGs) under practical considerations and common assumptions. An iterative integer programming (IP) framework is proposed, which drastically reduces the number of interventions required. Simulations over a broad range of graph topologies are used to assess the effectiveness of the proposed approach. Results show that the proposed adaptive IP approach achieves full causal graph recovery with fewer interventions and variable manipulations than random intervention baselines, while remaining flexible to practical constraints.

\end{abstract}

\section{Introduction}
\label{sec:intro}

Causal discovery aims to unravel the causal relationships among a set of variables by analyzing either observational or experimental data from the system under study \cite{Pearl2000,Spirtes2000}. In many scientific fields — including genomics, epidemiology, and social sciences — the ultimate goal is not merely to identify associations but also to determine which variables \emph{cause} changes in others. For instance, in genomics, researchers seek to understand how specific gene knockouts (a form of perfect intervention) affect downstream gene expression levels \cite{Sachs2005}. To formalize these causal relationships, \emph{directed acyclic graphs} (DAGs) have become a mainstay, in which vertices (nodes) represent random variables, while edges capture directional influences. Learning a causal DAG, often referred to as a \emph{causal Bayesian network} when combined with a probability distribution, typically requires not only observational data but also strategic experimentation or \emph{intervention} \cite{Spirtes2000,Cooper1999,HauserBuhlmann2012}. However, interventions may be costly (e.g., expensive lab experiments) or limited by practical or ethical constraints (e.g., restrictions on human experimentation), prompting researchers to seek methods that minimize the number and complexity of experiments while still guaranteeing that the underlying causal graph can be recovered \cite{TongKoller2001,Murphy2001,Eberhardt2008}.

\paragraph{DAGs and Causal DAGs.}
Mathematically, DAG \(D = (V, E)\) consists of a finite set of vertices \(V = \{1, 2, \ldots, N\}\) together with a set of directed edges \(E\). Each vertex represents a random variable, and a directional edge from vertex \(i\) to vertex \(j\) indicates that variable \(i\) has a direct influence on variable \(j\). Importantly, \(D\) contains no directed cycles. More formally, if there is a path \(i \to j \to \cdots \to k\) in \(D\), then there must not be any path from \(k\) back to \(i\). A DAG is often used to represent statistical dependencies in a factorized form:
\[
P(x_1, \ldots, x_N) \;=\; \prod_{i=1}^N P\bigl(x_i \,\big\vert\, \mathrm{pa}(i)\bigr),
\]
where \(\mathrm{pa}(i)\) denotes the set of parents (direct predecessors) of node \(i\). This factorization implies that each variable is independent of its non-descendants given its parents. When these edges are further interpreted as \emph{causal} relationships, and the probability distribution on \(V\) is assumed to satisfy the \emph{Causal Markov} and \emph{Faithfulness} assumptions, the DAG becomes a \emph{causal DAG} or a \emph{causal Bayesian network} \cite{Spirtes2000,Pearl2000}. Concretely, an edge \(i \to j\) in a causal DAG indicates that manipulating the value of variable \(i\) can alter the distribution of \(j\), all else being equal.

\paragraph{Markov Equivalence and Essential Graphs.}
When learning a causal graph from purely observational data (i.e., without interventions), a DAG is only identifiable \emph{up to a Markov equivalence class} (MEC) \cite{VermaPearl1990,Andersson1997}. This means that multiple DAGs can encode the same set of conditional independence relationships, thus it is impossible to distinguish between them based on observational data alone. Each equivalence class can be compactly represented by a \emph{completed partially directed acyclic graph} (\emph{CPDAG}) or \emph{essential graph}, where any edge directed in the CPDAG is oriented the same way in every member of that class, and undirected edges represent edges whose orientation remains ambiguous.

\paragraph{Types of Interventions.} Interventions in a causal system exogenously set variable values, breaking dependencies \cite{Pearl2000} to resolve causal ambiguities \cite{Spirtes2000,HeGeng2008}. Key types include:
\begin{itemize}
\item \textbf{Single-variable vs. Multi-variable:} Traditional single-variable interventions contrast with multi-variable approaches that accelerate discovery \cite{HauserBuhlmann2012,Eberhardt2007} and are necessary for capturing interaction effects \cite{BoxHunter2005}.
\item \textbf{Perfect (hard) vs. Imperfect (soft):} Perfect interventions fix values, removing parent influences (e.g., gene knockout). Imperfect interventions modify distributions while retaining some parental influence (e.g., gene knockdown) \cite{Korb2004,EatonMurphy2007}.
\end{itemize}

\paragraph{Strategies for Intervention Design.}
Conventionally, experimental design in statistics has focused on how best to assign treatment levels \emph{within} a predetermined set of experimental factors \cite{Fisher1935,BoxHunter2005}. Large-scale causal discovery poses an additional challenge: one must also decide \emph{which} variables to intervene on, and that choice typically depends on the partial knowledge that has accumulated about the causal graph \cite{TongKoller2001,Cooper1999,Eberhardt2008}. Because the number of possible intervention subsets grows exponentially with the number of variables, this decision problem quickly becomes computationally demanding.
Two broad approaches are employed. In a \textbf{fixed} or \textbf{passive} design, the entire sequence of experiments is planned in advance and executed exactly as specified, without regard to any information revealed along the way \cite{HauserBuhlmann2012,HeGeng2008,elrefaey2024causal}. While conceptually simple, a fixed schedule can be highly inefficient for large systems because it ignores opportunities to redirect effort toward the remaining uncertainties in the graph. By contrast, an \textbf{adaptive} or \textbf{active} design selects each new intervention only after the data from previous rounds have been analysed \cite{TongKoller2001,Murphy2001}. This feedback loop allows the experimenter to focus on the edges and orientations that are still ambiguous, and in many practical settings it dramatically reduces the total number of experiments required to recover the full causal structure.

\paragraph{Our Contributions.}
In this paper, we present an \emph{iterative method} for intervention design inline with adaptive intervention design for causal discovery. In this paper, we assume \emph{hard} interventions, where each intervened variable is forcibly set, independently of its direct causes.
We present a novel integer programming (IP) framework that can:
\begin{enumerate}
    \item \textbf{Select Informative Intervention Variables:} Our IP model determines which subset of variables to intervene on during each round of intervention experiment with the goal of maximizing the knowledge gained about the graph.
    \item \textbf{Integrate Logical Inference via Meek’s Rules:} After collecting interventional data and updating uncertain edges, we apply Meek’s orientation rules \cite{Meek1995} to propagate and resolve additional edges without additional experiments.
    \item \textbf{Integrate Practical Limitations as Constraints} The proposed IP is highly modular and offers great flexibility. It accommodates various practical considerations as constraints and can be easily extended to accommodate a number of scenarios.
\end{enumerate}

The remainder of this paper is organized as follows.
Section~\ref{sec:Related_Work} provides relevant work. Section~\ref{sec:methodology} formally states our problem setup and introduces the \emph{partially known graph} (\textsf{PKG}) representation of a causal DAG. It also details our proposed adaptive IP approach, and integration with Meek’s rules. We then report on our experimental simulations and findings in Section~\ref{sec:Experiments}, before concluding with a discussion of future directions.

\section{Related Work}
\label{sec:Related_Work}

The design of intervention experiments for the discovery of causal structures has gained significant attention in recent years, with numerous approaches proposed to efficiently uncover cause-and-effect relationships. This section reviews key contributions in this area, focusing on the assumptions, methodologies, and limitations of various methods.


A diverse range of approaches have been explored for active learning in causal discovery.
Decision-theoretic frameworks were employed by \cite{meganck2006learning}, incorporating intervention costs into their utility function and considering modular experiments where a single variable is targeted per intervention. They used an adaptive strategy.
\cite{eberhardt2008number} focused on establishing theoretical bounds for identifying causal relations, considering scenarios where any number of variables could be simultaneously intervened on, However, they did not explicitly model intervention costs or provide a specific algorithm for intervention selection.\cite{he2008active} developed algorithms for both fixed and adaptive intervention design, aiming to minimize the number of manipulated variables, thus implicitly limiting intervention set size but without cost or budget considerations. \cite{Eberhardt2008} proposed the OPTINTER algorithm, incorporating limits on intervention set size through a parameter and employing an adaptive strategy. Their paper did not explicitly model intervention costs but aimed to minimize the number of experiments. Using a fixed strategy, \cite{hauser2014two} introduced greedy algorithms, OPTSINGLE, which employs an adaptive strategy and limits interventions to single vertices, and another algorithm, OPTUNB, which allows interventions on multiple variables simultaneously without a specified limit. The focus was on minimizing the number of interventions rather than the cost associated with each intervention.
By drawing parallels to graph-theoretic concepts, \cite{hyttinen2013experiment} employed a fixed design strategy and implicitly addressed intervention costs by considering an objective of minimizing the number of intervenable variables in addition to the number of interventions. They provided algorithms to achieve this, considering worst-case (complete graph) scenarios which have limited practicality. \cite{hu2014randomized} proposed theoretical limits to using randomized adaptive strategies and argued that they improve upon existing deterministic methods, but they did not incorporate limits on the number of variables that could be intervened on per intervention or model in their derived bound. Moreover, they do not provide algorithms to attain those limits.
\cite{shanmugam2015learning} considered both fixed and adaptive strategies, explicitly incorporating limits on intervention set size while aiming to minimize the number of interventions. While their adaptive algorithms meet the derived bounds, they do not consider intervention costs. Both \cite{kocaoglu2017cost} and \cite{lindgren2018experimental} employ a fixed strategy and explicitly modeled intervention costs, with each variable having an associated cost. They formulated the problem of learning a causal graph as an integer program (IP), aiming to design a set of interventions with minimum total cost that can uniquely identify any causal graph with a given skeleton. The IP model was used to develop algorithms for specific graph structures, such as trees or clique trees. However, the cost structure is assumed to be additive. \cite{ghassami2019interventional} employed a fixed strategy focusing on single-variable interventions and the number of interventions that can be performed. They did not explicitly consider intervention costs.
\cite{porwal2022universal} establish a new universal lower bound on the number of single-node interventions required to fully orient an essential graph. The paper did not propose any new algorithm but rather focused on establishing theoretical bounds. The proposed lower bound applies to any algorithm, whether adaptive or fixed. Finally, \cite{elrefaey2024causal} allowed for explicit modeling of intervention (possibly nonlinear) costs and limits on intervention set size. They formulated a flexible IP model to determine the minimal set of interventions required for causal identifiability under a variety of practical scenarios. However, their approach was for a fixed design strategy.

\section{Methodology}
\label{sec:methodology}

We address the \emph{adaptive design of interventions} for causal structure learning, where, based on the current knowledge of causal structure, each intervention aims to reveal as much additional causal structure knowledge as possible.

\paragraph{Assumptions.} The intervention-based causal structure discovery and Meek's rules can be used because the following standard assumptions of causal graphs are applied:
\begin{enumerate}
    \item \textbf{Faithfulness:} All statistical independencies observed in the data correspond to d-separations \cite{pearl2009causality} in the true DAG.
    \item \textbf{Sufficiency:} All relevant variables are measured, so there are no unobserved confounders.
    \item \textbf{Conditional Independence Oracle:} We assume access to (or a reliable approximation of) a perfect conditional independence test.
\end{enumerate}

\noindent Under these assumptions, and building upon the previous work of \cite{hu2014randomized, kocaoglu2017cost, lindgren2018experimental, elrefaey2024causal}, we solve an IP model, \textsc{Adaptive\_IP}, to adaptively choose the set of variables to intervene on. We choose the set which maximizes the knowledge learnt about the true causal DAG, \(G^*\), subject to some constraints such as a fixed intervention budget and/or a limit on how many variables can be intervened on in each experiment. After each intervention, the \textsf{PKG} is updated, and Meek's rules are applied to logically propagate further orientations. This approach is implemented as described in Algorithm \ref{alg:adaptiveip_meek}. Proofs of correctness and finite convergence of this algorithm can be found in Appendix \ref{proofs}.

\begin{algorithm}[h]
   \caption{Iterative Orientation via Adaptive IP and Meek's Rules}
   \label{alg:adaptiveip_meek}
   \small
\begin{algorithmic}[0]  
   \STATE {\bfseries Input:} A \textsf{PKG} with edges partitioned into: ${E}_{\text{Known}}$ (oriented edges), ${E}_{\text{Unknown}}$ (undecided), ${E}_{\text{semi-directed}}$ (semi-directed), ${E}_{\text{Adjacent}}$ (undirected).
   \STATE {\bfseries Goal:} Fully orient all edges into ${E}_{\text{Known}}$.

   \REPEAT
     \STATE \textbf{Run} \textsc{Adaptive\_IP} on \textsf{PKG} to propose orientations:
        \STATE \quad $\{\text{actions}\} \leftarrow \textsc{Adaptive\_IP}(\text{\textsf{PKG}})$.
     \STATE \textbf{Update} \textsf{PKG} edge sets based on the chosen solution and intervention results.
     \STATE \textbf{Apply} \textsc{Meek's Rules} to orient additional edges:
        \STATE \quad $\textsc{MeeksRules}(\text{\textsf{PKG}}) \rightarrow \text{\textsf{PKG}}'$.
     \STATE \textbf{Replace} \textsf{PKG} with the updated \textsf{PKG}$'$.
   \UNTIL{all edges are placed into ${E}_{\text{Known}}$}

   \STATE {\bfseries Output:} Fully oriented graph in ${E}_{\text{Known}}$.
\end{algorithmic}
\end{algorithm}

\paragraph{Partially Known Graph. }
We maintain a \textsf{PKG}  \cite{Eberhardt2008} that evolves over iterations. The vertices in the \textsf{PKG} are the same as the causal DAG that is being learned (\({V} = \{1, 2, \ldots, N\}\)). An edge between a pair of vertices on this graph belongs to exactly one of four relationships, which are
\begin{itemize}
    \item \textbf{Unknown (\(E_{\text{Unknown}}\)):} Edges for which we do not know if they exist or not (fully unknown node adjacency and edge orientation).
    \item \textbf{Adjacent (\(E_{\text{Adjacent}}\)):} Edges whose node adjacency is established (the edge is present), but the orientation is unknown.
    \item \textbf{Semi-directed (\(E_{\text{Semi-directed}}\)):} Edges that may either exist in a known direction or do \emph{not} exist at all.
    \item \textbf{Known (\(E_{\text{Known}}\)):} Edges with both node adjacency and edge orientation fully confirmed.
\end{itemize}
After each intervention, the \textsf{PKG} is revised based on newly identified adjacencies and orientations as shown in section \ref{trans_rules}.

We then apply Meek's rules to further propagate orientations that do not require additional interventions. This \textsf{PKG} formalism allows us to encode a variety of \emph{background knowledge} about the causal structure - including known or forbidden edges - throughout the discovery process. Note that the well-studied \emph{essential graph} is a special case of the \textsf{PKG} that only involves \(E_{\text{Adjacent}}\) and \(E_{\text{Known}}\).

\paragraph{Meek's Rules }

\begin{figure}[ht]
\begin{center}
\centerline{\includegraphics[width=\linewidth,height=0.4\textwidth]{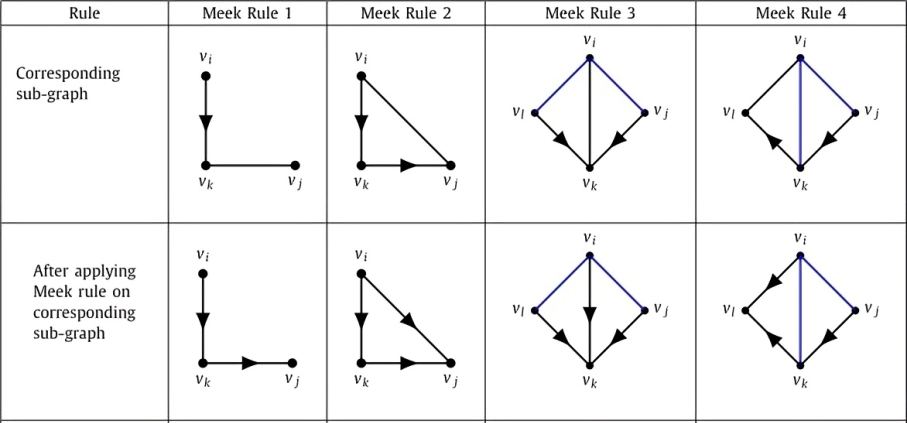}}
\caption{Meek Rules of Edge Orientation}
\label{meek}
\end{center}
\end{figure}

For a partially directed graph \({G} = ({V}, E)\), let \((v_i - v_j)\) denote an undirected edge and \((v_i \to v_j)\) a directed edge. If \((v_i, v_j) \notin E\), then \(v_i\) and \(v_j\) share no edge (directed or undirected). Meek's rules (R1)-(R4) orient edges as follows (Figure \ref{meek}):

\begin{itemize}
\item[\textbf{(R1)}]
If \(\exists\,v_i\) with \((v_i \to v_k) \in E\) and \((v_i, v_j) \notin E\), orient \((v_k - v_j)\) as \((v_k \to v_j)\).
\item[\textbf{(R2)}]
If \(\exists\,v_k\) with \((v_i \to v_k) \in E\) and \((v_k \to v_j) \in E\), orient \((v_i - v_j)\) as \((v_i \to v_j)\).
\item[\textbf{(R3)}]
If \(\exists\,v_k,v_\ell\) with 
\((v_i - v_k)\in E,\,
(v_i - v_\ell)\in E,\,
(v_k \to v_j)\in E,\,
(v_\ell \to v_j)\in E\),
and \((v_k, v_\ell)\notin E\),
then orient \((v_i - v_j)\) as \((v_i \to v_j)\).
\item[\textbf{(R4)}]
If \(\exists\,v_b,v_d\) with
\((v_b \to v_j)\in E,\,
(v_i - v_d)\in E,\,
(v_i - v_b)\in E,\,
(v_d \to v_b)\in E\),
and \((v_j, v_d)\notin E\),
then orient \((v_i - v_j)\) as \((v_i \to v_j)\).
\end{itemize}

By iterating these rules until none apply, one can propagate existing directions to newly orient edges without additional interventions or data collection.

\subsection{IP Model Formulation}
We formulate the \textsc{Adaptive\_IP} model as follows:

\subsubsection{Notation}

\begin{description}[labelwidth=1.5cm, labelsep=0.2cm, leftmargin=1.7cm]
    \item[$B$] 
    Maximum allowed budget for one intervention (iteration).
    
    \item[$k_{max}$] 
    Maximum allowed number of variables that can be intervened in simultaneously.
    
    \item[$\mathcal{X}$]
    A subset of vertices $V$ (i.e., $\mathcal{X} \subseteq V$) that are \emph{viable} for intervention. Formally,
      $\mathcal{X} = \{\, i \in V \;|\; i \text{ is incident on an edge in } E_{\text{Unknown}} \cup E_{\text{Adjacent}} \cup E_{\text{Semi-directed}} \}$.

    \item[$X_i$] 
    A binary variable indicating whether or not vertex $i$ is selected to be intervened on ($X_i = 1$ or $X_i = 0$, respectively). Only vertices in $\mathcal{X}$ can be potentially selected to have $X_i = 1$.

    \item[$CI_i$] 
    Cost of intervening on variable $i$.

    \item[$CO_i$] 
    Cost of observing variable $i$.

    \item[$O_{ij}$]
    A binary variable indicating whether edge $(i,j)$ is \emph{subject to an orientation test}. In other words, $O_{ij} = 1$ iff vertex $i$ is intervened on ($X_i = 1$) but vertex $j$ is \emph{not} ($X_j = 0$). This variable is only defined for edges in $E_{\text{Unknown}} \cup E_{\text{Semi-directed}} \cup E_{\text{Adjacent}}$.

    \item[$A_{ij}$]
    A binary variable indicating whether edge $(i,j)$ is \emph{subject to an adjacency test}. Specifically, $A_{ij} = 1$ iff neither vertex $i$ nor vertex $j$ is intervened on ($X_i = 0$ and $X_j = 0$). This variable is only defined for edges in $E_{\text{Unknown}} \cup E_{\text{Semi-directed}}$.

    \item[$IDU_{ij}$]
    A binary variable indicating that an edge $(i,j)$ in $E_{\text{Unknown}}$ has been \emph{updated} by the current solution. That is, $IDU_{ij} = 1$ if the model decides to move $(i,j)$ from $E_{\text{Unknown}}$ into another category of edges.

    \item[$IDS_{ij}$]
    A binary variable indicating that an edge $(i,j)$ in $E_{\text{Semi-directed}}$ has been \emph{updated}.

    \item[$IDA_{ij}$]
    A binary variable indicating that an edge $(i,j)$ in $E_{\text{Adjacent}}$ has been \emph{updated}.

\end{description}

\subsubsection{Objective Function}
\begin{equation}
\label{o.f.}
\text{maximize} \quad \sum_{(i,j) \in E_{\text{Unknown}}} IDU_{ij} + \sum_{(i,j) \in E_{\text{Semi-directed}}} IDS_{ij} + \sum_{(i,j) \in E_{\text{Adjacent}}} IDA_{ij}
\end{equation}

\subsubsection{Constraints}

\begin{equation}
\label{budget}
\sum_{i \in \mathcal{X}} \bigl(CI_i \cdot X_i + CO_i \cdot (1-X_i)\bigr) \leq B 
\end{equation}
\begin{equation}
\label{kmax}
\sum_{i \in \mathcal{X}} X_i \leq k_{\text{max}} 
\end{equation}
\begin{equation}
\label{o1}
O_{ij} \leq X_i \quad\quad \forall (i,j) \in E_{\text{Unknown}} \cup E_{\text{Semi-directed}} \cup E_{\text{Adjacent}}
\end{equation}
\begin{equation}
\label{o2}
O_{ij} \leq 1 - X_j \quad\quad \forall (i,j) \in E_{\text{Unknown}} \cup E_{\text{Semi-directed}} \cup E_{\text{Adjacent}}
\end{equation}
\begin{equation}
\label{a1}
A_{ij} \leq 1 - X_i \quad\quad \forall (i,j) \in E_{\text{Unknown}} \cup E_{\text{Semi-directed}}
\end{equation}
\begin{equation}
\label{a2}
A_{ij} \leq 1 - X_j \quad\quad \forall (i,j) \in E_{\text{Unknown}} \cup E_{\text{Semi-directed}}
\end{equation}
\begin{equation}
\label{IDU}
IDU_{ij} \leq O_{ij} + O_{ji} + A_{ij} \quad\quad \forall (i,j) \in E_{\text{Unknown}}
\end{equation}
\begin{equation}
\label{IDS}
IDS_{ij} \leq O_{ij} + A_{ij} \quad\quad \forall (i,j) \in E_{\text{Semi-directed}}
\end{equation}
\begin{equation}
\label{IDA}
IDA_{ij} \leq O_{ij} + O_{ji} \quad\quad \forall (i,j) \in E_{\text{Adjacent}}
\end{equation}
\begin{equation}
\label{binary}
X_i, O_{ij}, A_{ij}, IDU_{ij}, IDS_{ij}, IDA_{ij} \in \{0,1\}
\end{equation}

The budget constraint (Eq. (\ref{budget})) restricts the combined cost of interventions and observations to remain within the allowed budget \(B\), accounting for the intervention cost (\(CI_i\)) and observation cost (\(CO_i\)) for each variable. The next constraint (Eq. (\ref{kmax})) ensures that at most \(k_{\text{max}}\) variables can be intervened on simultaneously, limiting the number of active manipulations per intervention. Orientation tests (Eqs. (\ref{o1}) and (\ref{o2})) are defined such that an edge \((i,j)\) can only be subject to an orientation test if the source vertex \(i\) is intervened on (\(X_i = 1\)) and the target vertex \(j\) is not (\(X_j = 0\)). Similarly, adjacency tests (Eqs. (\ref{a1}) and (\ref{a2})) require that neither vertex of an edge is intervened on (\(X_i = 0\) and \(X_j = 0\)) for the test to apply. The update constraints (Eqs. (\ref{IDU}), (\ref{IDS}), and (\ref{IDA})) ensure that an edge in \(E_{\text{Unknown}}, E_{\text{Semi-directed}},\) or \(E_{\text{Adjacent}}\) can only be updated if it is subject to the relevant orientation or adjacency tests. Finally, the binary constraint (Eq. (\ref{binary})) forces all decision variables to be binary. We wish to maximize the sum of edge updates given by objective \ref{o.f.}. 

The model is quite modular and offers flexibility to accommodate various alternative objectives and constraints. We showcase a wide variety of model extensions in Appendix \ref{model_extensions}.

\subsection{Transition Rules for Edge Sets}
\label{trans_rules}

The following rules govern how edges move between the sets \(E_{\text{Unknown}}, E_{\text{Semi-directed}}, E_{\text{Adjacent}},\) and \(E_{\text{Known}}\) based on the results of interventions and tests. These rules are summarized in Table \ref{tab:transition_crosstab}.

\begin{table}[h]
\caption{Conditions and actions for edge transitions between edge sets.}
\small
\centering
\begin{tabular}{p{2.2cm}p{3.2cm}p{3.2cm}p{3.2cm}}
\toprule
\textbf{From \textbackslash To} & \(\mathbf{E_{\text{Known}}}\) & \(\mathbf{E_{\text{Semi-directed}}}\) & \(\mathbf{E_{\text{Adjacent}}}\)\\
\midrule
\(\mathbf{E_{\text{Unknown}}}\) & 
\(O_{ij} = 1\) \& presence. & 
\(O_{ij} = 1\) \& absence. & 
\(A_{ij} = 1\) \& presence. \\ 

\(\mathbf{E_{\text{Semi-directed}}}\) & 
(\(O_{ij} = 1\) or \(A_{ij} = 1\)) \& presence. & 
 & 
 \\
\(\mathbf{E_{\text{Adjacent}}}\) & 
(\(O_{ij} = 1\) or \(O_{ji} = 1\)) \& presence. & 
 & 
 \\ 
 \bottomrule
\end{tabular}
\\[1.5ex]
\label{tab:transition_crosstab}
\end{table}
\subparagraph{From \(E_{\text{Unknown}}\):}
 If the edge \((i, j)\) is subject to an orientation test (\(O_{ij} = 1\)) and the intervention reveals the edge to be present, then the edge is added to {\(E_{\text{Known}}\) in the tested direction. Otherwise, if it is absent, it is added to \(E_{\text{Semi-directed}}\) in the \textit{reverse direction \((j \to i)\)}. If the edge \((i, j)\) is subject to an adjacency test (\(A_{ij} = 1\)) and the intervention reveals that its present, the edge is moved to \(E_{\text{Adjacent}}\). Once moved into another edge set, the edge is removed from \(E_{\text{Unknown}}\). Also, if the edge is subject to an adjacency test and the intervention reveals that the edge is absent, it is removed from \(E_{\text{Unknown}}\).
\subparagraph{From \(E_{\text{Semi-directed}}\):}
If the edge \((i, j)\) is subject to an orientation test or an adjacency test, and the intervention reveals the edge to be present, it is moved to \(E_{\text{Known}}\) with its resolved direction.
The edge is then removed from \(E_{\text{Semi-directed}}\) and likewise if the intervention reveals it to be absent.
\subparagraph{From \(E_{\text{Adjacent}}\):}
If the edge \((i, j)\) is subject to an orientation test and the intervention reveals the edge to be present, it is moved to \(E_{\text{Known}}\). The edge is then removed from \(E_{\text{Adjacent}}\).

\section{Experiments}
\label{sec:Experiments}

To evaluate the performance of Algorithm \ref{alg:adaptiveip_meek}, we conduct a series of simulations using synthetic graphs as well as real graphs. The primary goal of these experiments is to assess the model's ability to learn the true DAG under different conditions such as graph sizes, edge densities, and intervention size limits. While the algorithm can work for any initial \textsf{PKG}, for the purpose of this study, the initial \textsf{PKG} is made to be the correct \emph{essential} graph of the true graph.

\subsection{Synthetic Graphs}

For each simulation, we generate 50 random DAGs using the $\mathcal{G}(N,p)$ Erdős-Rényi model. This model generates an undirected graph with $N$ nodes, where each edge is created independently with a probability $p$. The edges are then directed according to the identity topological ordering of the nodes to ensure acyclicity.
This is done for the set of parameters: \(\{(N,p) \ |\ N \in \{3, 4, 8, 16, 24, 32, 48, 64, 96, 128, 256\},\ p \in \{0.05, 0.2, 0.5, 0.7, 0.95\}\}\).

\subsection{Real Graphs}

Algorithm \ref{alg:adaptiveip_meek} was tested against 9 real graphs. The graphs have different sizes, densities, and degree distributions. Table \ref{table:rgs} summarizes the structural statistics of each graph.

\begin{table}[h]
  \centering
  \small
  \caption{Structural statistics of benchmark Graphs}
  \begin{tabular}{lrrrrrrr}
    \toprule
    Network  & Nodes & Edges & Min\ deg.& Avg.\ deg.& Max\ deg.& Stdev.\ deg.& \# v-structures\\
    \midrule
    asia        &   8 &   8 & 1 & 2.00 &   4 &  0.93 &   2\\
    sachs       &  11 &  17 & 2 & 3.09 &   7 &  1.64 &   0\\
    insurance   &  27 &  52 & 1 & 3.85 &   9 &  2.03 &  23\\
    alarm       &  37 &  46 & 1 & 2.49 &   6 &  1.35 &  24\\
    hailfinder  &  56 &  66 & 1 & 2.36 &  17 &  2.40 &  34\\
    win95pts    &  76 & 112 & 1 & 2.95 &  10 &  2.01 & 129\\
    pathfinder  & 109 & 195 & 1 & 3.58 & 106 & 10.16 &  16\\
    andes       & 223 & 338 & 0 & 3.03 &  12 &  1.87 & 313\\
    link        & 724 &1125 & 0 & 3.11 &  17 &  2.49 & 821\\
    \bottomrule
  \end{tabular}
  \label{table:rgs}
\end{table}

\subsubsection{Comparison Parameters and Metrics}

The following parameters are varied across the simulations:
\begin{itemize}
    \item \textbf{Maximum Interventions ($k_{max}$)}: We limited the maximum number of simultaneous interventions in each iteration to $k_{max} \in \{1, 2, 4, 6\}$.
    \item \textbf{Methods}: The performances of two methods are compared:
        \begin{itemize}
            \item \textbf{`r'}: A variable subset, $\mathcal{S}$ $(|\mathcal{S}| \leq k_{max})$, is chosen at random at each iteration. 
            \item \textbf{`IP':}
            \textsc{Adaptive\_IP}
        \end{itemize}
\end{itemize}

The following metrics are recorded from each experiment:
\begin{itemize}
    \item \textbf{Number of Iterations}: The number of iterations required to complete the causal discovery process.
    \item \textbf{Total Variables Intervened On}: The total number of variables that are intervened on across all iterations.
\end{itemize}

For the purpose of this study, we assume that the cost of interventions is negligible and hence constraint \ref{budget} is not added to the model. Additionally, for instances where multiple optimal solutions to the \textsc{Adaptive\_IP} model exist, a random one is chosen.
Simulations are implemented in Python using the NetworkX \cite{hagberg2008exploring} library for synthetic graph generation and the Gurobi solver for solving the \textsc{Adaptive\_IP} model. 

\subsection{Results and Discussion}
\label{sec:results-discussion}

Figures~\ref{fig:addl_interventions_sg} and~\ref{fig:addl_interventions_rg}
plot, for every experimental setting, the quantity
\[
\Delta_{rounds} \;=\;
(\#\text{rounds}_{\textit{r}}) -
(\#\text{rounds}_{\textit{Adaptive\_IP}})
\] 
that is the \emph{additional} intervention iterations required by the random
baseline.  Hence positive values indicate an unambiguous advantage for the
optimized strategy, while \(\Delta_{rounds}=0\) corresponds to parity.
Across \emph{all} experiments median \(\Delta_{rounds}\ge0\), confirming that
Adaptive\_IP never performs worse than random selection and often yields
substantial savings.

\paragraph{Synthetic Erdős-Rényi graphs}
(Figure~\ref{fig:addl_interventions_sg}).
The magnitude of the gap depends strongly on graph size~\(N\), edge
probability~\(p\), and the per-round intervention budget~\(k_{\max}\):

\textbf{$k_{max}$ dependence} 
For each facet, when \(k_{\max}=6\), the distributions sit highest for (\(0.20\le p\le0.70\)): when many variables can
be perturbed in parallel, Adaptive\_IP can exploit the large combinatorial space to pick a particularly informative six-tuple, whereas random choice often squanders part of that budget. In contrast, for the dense regime (\(p=0.95\)) the benefit is \emph{largest} when only \emph{one} variable can be manipulated.  There the baseline’s single random knock-out is frequently uninformative, whereas the IP solver consistently targets a key hub; as \(k_{\max}\) increases the gap collapses because almost any two or more perturbations break enough cycles to let Meek’s rules finish the job.

\textbf{Density dependence}  
Very sparse graphs (\(p=0.05\)) and very dense graphs (\(p=0.95\)) exhibit the broadest inter-quartile ranges—up to twelve rounds in the sparse case and more than sixty rounds in the dense case—because each contains few v-structures, leaving Meek’s propagation with little to do; progress hinges on how informative the chosen interventions are. Intermediate densities (\(0.20\le p\le0.70\)) feature many v-structures, hence once a handful of orientations are fixed Meek’s rules cascade through the graph and even randomly chosen interventions achieve respectable
coverage; the median advantage in these panels rarely exceeds ten rounds.

\begin{figure}[h]
\includegraphics[width=\linewidth,height=0.4\linewidth]{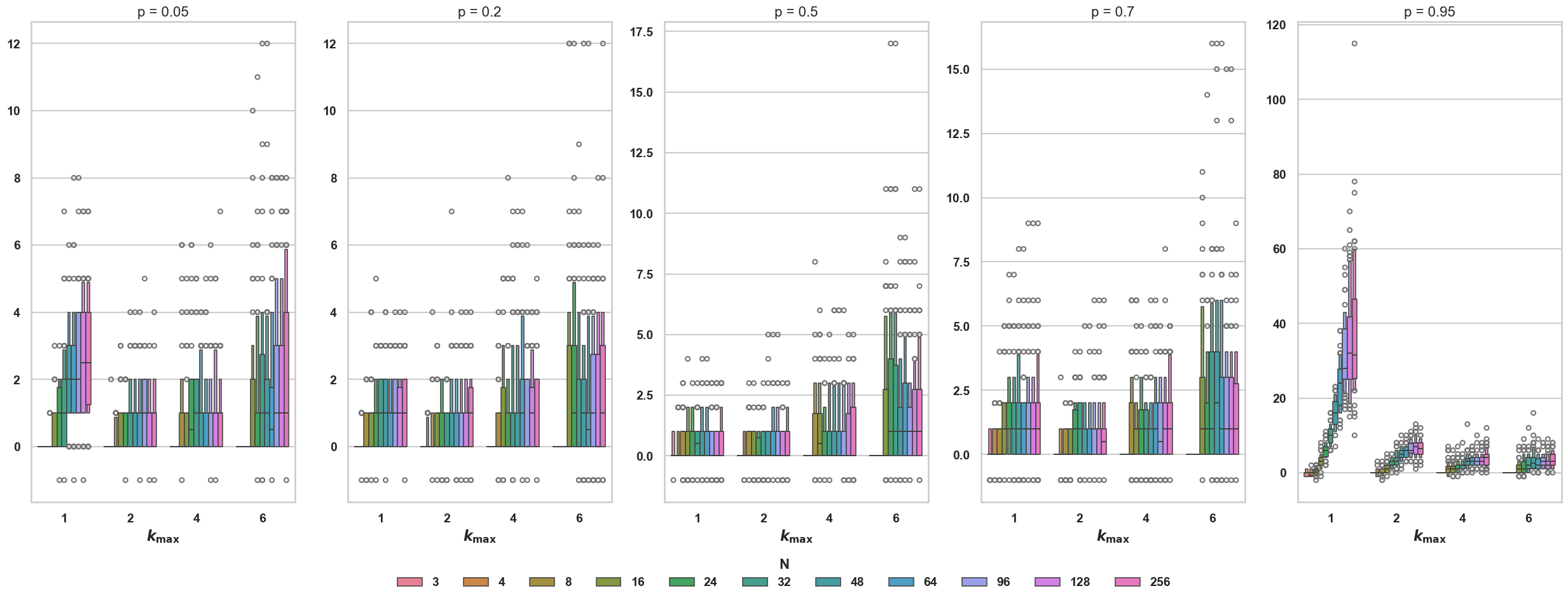}
  \caption{%
    \textbf{Number of additional intervention rounds required when using method
    $r$ versus \textit{Adaptive\_IP}.}
    Each panel corresponds to a graph-edge probability $p$; bars show the
    distribution of $r-\mathrm{IP}$ across seeds for different intervention
    limits $k_{\max}$ and network sizes $N$.}
  \label{fig:addl_interventions_sg}
\end{figure}
\vspace{-\baselineskip}

\paragraph{Real benchmark graphs.}
Figure~\ref{fig:addl_interventions_rg} compares the extra intervention
rounds required by the random baseline across the nine real networks whose
structural statistics are listed in Table~\ref{table:rgs}.
interventions are allowed. The extent of the advantage of the adaptive IP framework differs sharply from graph
to graph and tracks well with the topological indicators in
Table~\ref{table:rgs}.
In \textit{pathfinder} the median saving exceeds forty iterations; the
network contains a central hub node (maximum degree~106) and progress hinges on perturbing that specific node—something the IP
identifies reliably but a random pick rarely does.
\textit{hailfinder} and \textit{insurance} also benefit markedly
(\(\Delta_{rounds}\approx 4\!-\!8\)) because their degree distributions are
skewed and they possess only a moderate number of v-structures (34 and 23, respectively), leaving Meek’s rules with limited reach unless the right edges are experimentally oriented.
At the other extreme, \textit{link} shows virtually no median improvement
even at \(k_{\max}=1\); with 821 v-structures, many edges are already oriented in the MEC,  and the random baseline performs as well as the optimized selection.
Small sparse graphs such as \textit{asia} and \textit{sachs} exhibit the same phenomenon on a lesser scale—the uncertainty is so limited that even a
naïve intervention suffices to complete the discovery in just one or two additional rounds. Taken together, the real-graph results highlight \textbf{when} an adaptive IP planner is worth its computational cost: namely, in
networks whose orientation hinges on a few critical hubs or on edges that
cannot be resolved by Meek’s rules alone, and in experimental regimes where
only one variable can be perturbed at a time.
When many v-structures are present—or when even a modest multi-variable
budget is available—the simpler random strategy approaches parity.

Figures \ref{fig:addl_variables_sg} and \ref{fig:addl_variables_rg} demonstrate similar results for the number of additional number of variable-manipulations required for graph identification by the random baseline. 

These simulations aim to validate the feasibility of Algorithm \ref{alg:adaptiveip_meek} under the assumptions listed in Section \ref{sec:methodology}. While the focus of the simulations conducted was on the efficiency of the selected interventions in learnine the true graph, we emphasize that the reported benefit of the framework is attributed to the nature of Objective \ref{o.f.}. There exists other objectives (see Section \ref{sec:objs}) that can outperform Objective \ref{o.f.} (which is myopic), albeit, tend to be more computationally expensive to evaluate. Furthermore, we emphasize that the main contribution of this work is its offered modularity and flexibility. For instance, the same framework can be utilized with an alternative objective while accommodating practical constraints such as intervention costs and experimentation budget, limits on the number of variables that can be intervened on per iteration, prior knowledge, and other considerations by simply replacing Objective \ref{o.f.} with another. No other adaptive intervention design strategy offers this level of flexibility and hence it was difficult to directly compare the proposed framework to methods found in the literature.

\section{Conclusion}
This paper presented a novel adaptive intervention strategy for causal discovery built upon iterative Integer Programming. A primary contribution is the IP framework's flexibility, allowing seamless integration of practical constraints and diverse objectives, which is challenging for existing methods. Simulations demonstrated performance improvements over random baselines, and theoretical analysis confirmed finite convergence and correctness under assumptions including Causal Sufficiency and a Perfect Independence Oracle. The inherent modularity of our approach paves the way for crucial future research, particularly focused on relaxing these assumptions to robustly handle unobserved confounders and work with finite, noisy real-world data.

\clearpage
\bibliographystyle{ieeetr}
\bibliography{example_paper}

\clearpage
\appendix
\section{Additional Simulation Figures}

\begin{figure}[ht]
  \includegraphics[width=\textwidth]{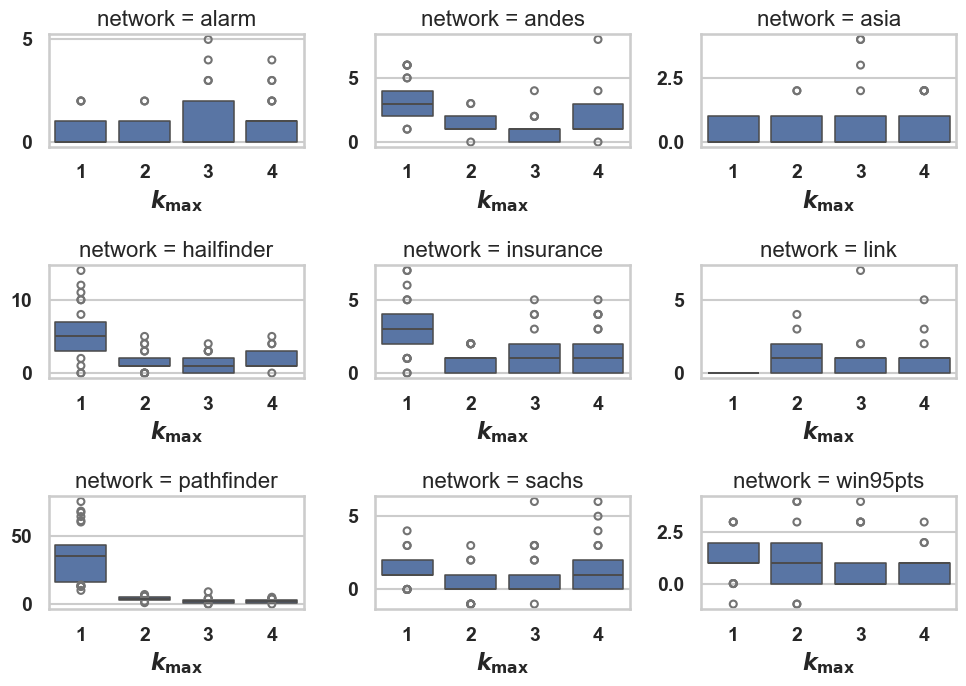}
  \caption{%
    \textbf{Number of additional intervention rounds required when using method
    $r$ versus \textit{Adaptive\_IP}.}
    Each panel corresponds to one of 9 networks; bars show the
    distribution of $r-\mathrm{IP}$ across seeds for different intervention
    limits $k_{\max}$.}
  \label{fig:addl_interventions_rg}
\end{figure}

\begin{figure}[ht]
  \includegraphics[width=\linewidth]{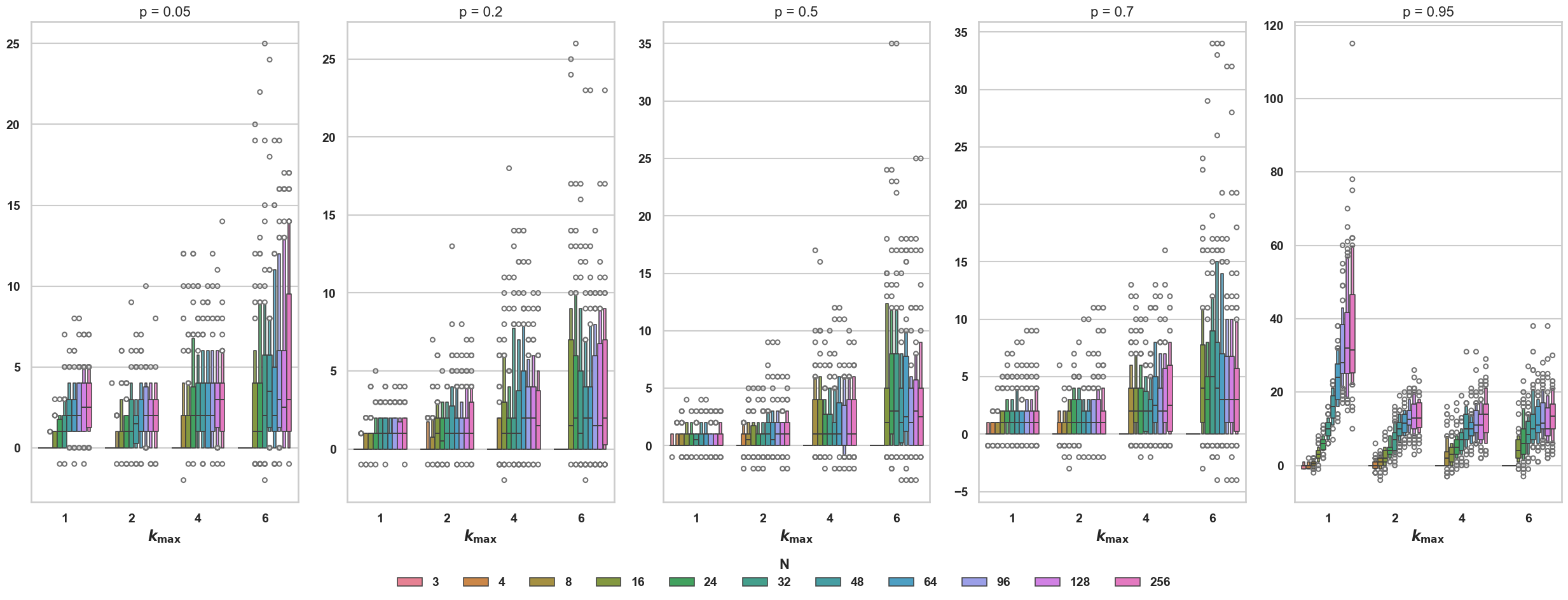}
  \caption{%
    \textbf{Number of additional variables manipulations required when using method
    $r$ versus \textit{Adaptive\_IP}.}
    Each panel corresponds to a graph-edge probability $p$; bars show the
    distribution of $r-\mathrm{IP}$ across seeds for different intervention
    limits $k_{\max}$ and network sizes $N$.}
  \label{fig:addl_variables_sg}
\end{figure}

\begin{figure}[ht]
  \includegraphics[width=\textwidth]{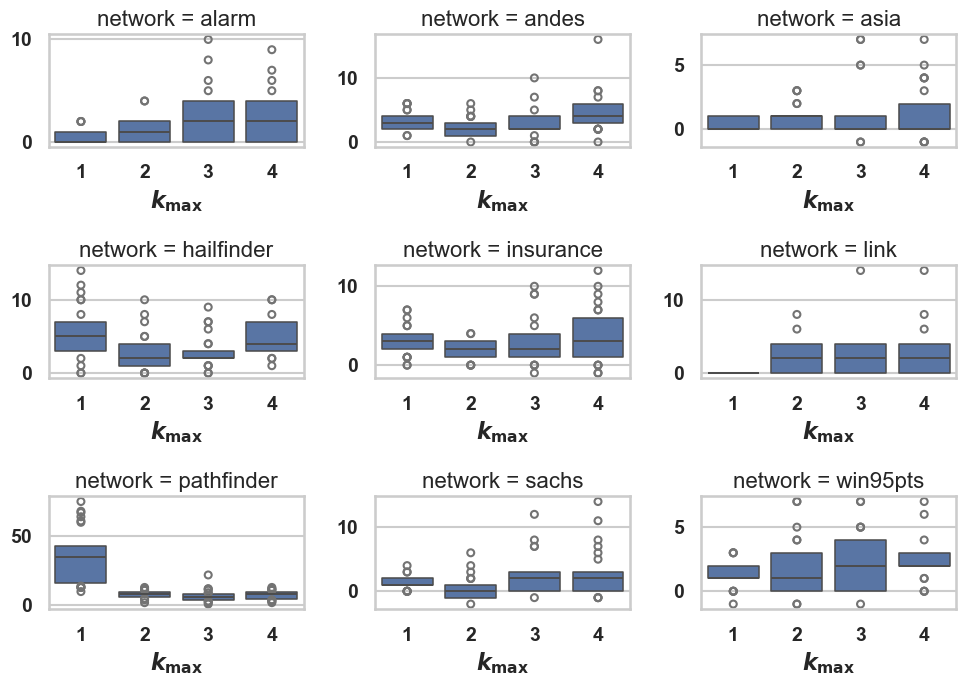}
  \caption{%
    \textbf{Number of additional variable manipulations required when using method
    $r$ versus \textit{Adaptive\_IP}.}
    Each panel corresponds to one of 9 networks; bars show the
    distribution of $r-\mathrm{IP}$ across seeds for different intervention
    limits $k_{\max}$.}
  \label{fig:addl_variables_rg}
\end{figure}

\clearpage
\section{Proofs \& Theoretical Analysis} 
\label{proofs}

In this appendix, we provide theoretical guarantees for the proposed adaptive intervention design algorithm (Algorithm \ref{alg:adaptiveip_meek}). These results hold under the standard assumptions commonly made in causal discovery from interventions:

\begin{enumerate}[label=\textbf{A\arabic*.}, itemsep=1pt, topsep=3pt]
    \item \textbf{Causal Sufficiency:} The set of observed variables \(V\) includes all common causes of pairs of variables in \(V\).
    \item \textbf{Faithfulness:} The probability distribution \(P\) over \(V\) is faithful to the true causal Directed Acyclic Graph (DAG) \(G^* = (V, E^*)\). All conditional independencies in \(P\) are entailed by the structure of \(G^*\) via d-separation.
    \item \textbf{Perfect Interventions:} An intervention on a set \(I \subseteq V\) sets the values of variables in \(I\), severing all incoming edges to nodes in \(I\) in \(G^*\), resulting in a manipulated graph \(G^*_{\text{do}(I)}\) and distribution \(P_{\text{do}(I)}\).
    \item \textbf{Perfect Conditional Independence (CI) Oracle:} There exists an oracle that correctly determines any conditional independence relationship \(X \perp Y \mid Z\) in any distribution \(P_{\text{do}(I)}\) generated by an intervention \(I \subseteq V\) on the true causal model.
    \item \textbf{Acyclicity:} The true underlying causal structure \(G^*\) is a DAG.
    \item \textbf{Finitude:} The set of variables \(V\) is finite, \(|V| = N\).
\end{enumerate}

We analyze the state of knowledge using the \textsf{PKG}, represented by the partition of all possible ordered pairs \((i, j)\) with \(i \neq j\) into the sets \(E_{\text{Known}}\), \(E_{\text{Adjacent}}\), \(E_{\text{Semi-directed}}\), and \(E_{\text{Unknown}}\). (Note: \(E_{\text{Adjacent}}\) contains undirected edges, formally pairs \(\{i, j\}\), but we can represent it with pairs \((i, j)\) and \((j, i)\) having linked status). Let \(S_t = (\textsf{PKG}_t)\) denote the state at the beginning of iteration \(t\).

\begin{theorem}[Finite Convergence]
\label{thm:convergence}
Under Assumptions A1-A6, Algorithm \ref{alg:adaptiveip_meek} terminates in a finite number of iterations.
\end{theorem}

\begin{proof}
Let \(S_t = (\textsf{PKG}_t)\) be the state (the configuration of edge sets \(E_{\text{Known}}^{(t)}, E_{\text{Adjacent}}^{(t)}, E_{\text{Semi-directed}}^{(t)}, E_{\text{Unknown}}^{(t)}\)) at the start of iteration \(t\). The total number of ordered pairs \((i, j)\) with \(i \neq j\) is finite, \(N(N-1)\). Since each pair must belong to one of the four sets (or represent a confirmed absence of an edge), the total number of possible states \(S_t\) is finite.

Define a measure of ambiguity at state \(S_t\) as the number of ordered pairs whose status is not fully resolved (i.e., not in \(E_{\text{Known}}\) and not confirmed absent):
\[ M(S_t) = |\{(i, j) \mid (i \to j) \in E_{\text{Semi-directed}}^{(t)}\}| + |\{(i, j) \mid \{i, j\} \in E_{\text{Adjacent}}^{(t)}\}| + |\{(i, j) \mid \{i, j\} \in E_{\text{Unknown}}^{(t)}\}| \]
Note that \(M(S_t)\) is a non-negative integer. \(M(S_t) = 0\) if and only if all relationships are either in \(E_{\text{Known}}\) or confirmed absent, which is the termination condition.

Consider the transition from state \(S_t\) to \(S_{t+1}\) in one iteration of Algorithm \ref{alg:adaptiveip_meek}.
\begin{enumerate}
    \item \textbf{Intervention Selection \& Update:} The \textsc{Adaptive\_IP} model is solved, yielding an intervention \(I^*\). Let the optimal objective value be \(Z^*\). Tests (\(O_{ij}, A_{ij}\)) associated with \(I^*\) are performed using the perfect oracle. The \textsf{PKG} is updated according to the Transition Rules (Table \ref{tab:transition_crosstab}), resulting in an intermediate state \(S'_t\).
        \begin{itemize}
            \item Each application of a transition rule either moves a pair \((i, j)\) or \(\{i, j\}\) to a state of equal or lesser ambiguity (e.g., \(E_{\text{Unknown}} \to E_{\text{Adjacent}}\), \(E_{\text{Adjacent}} \to E_{\text{Known}}\)) or confirms absence (removing it from ambiguity count).
            \item Crucially, no rule moves an edge to a state of strictly greater ambiguity (e.g., \(E_{\text{Known}} \to E_{\text{Adjacent}}\)).
            \item If \(Z^* > 0\), at least one test was performed on an edge in an uncertain state. The perfect oracle yields a definitive outcome, triggering a transition rule. This rule application either moves the edge closer to \(E_{\text{Known}}\) or confirms its absence. In either case, the pair \((i, j)\) or \(\{i, j\}\) involved contributes less or equally to the ambiguity measure \(M\) in \(S'_t\) compared to \(S_t\). At least one relationship's status changes.
        \end{itemize}
        Therefore, \(M(S'_t) \leq M(S_t)\).
    \item \textbf{Meek's Rules Application:} Meek's rules are applied to \(S'_t\) to obtain \(S_{t+1}\). Meek's rules only orient edges in \(E_{\text{Adjacent}}\), moving them to \(E_{\text{Known}}\). They do not add edges or change edges in \(E_{\text{Known}}, E_{\text{Semi-directed}}, E_{\text{Unknown}}\).
        \begin{itemize}
            \item If Meek's rules orient \(k \ge 0\) edges, the ambiguity measure decreases by \(k\) (since each orientation moves a pair from \(E_{\text{Adjacent}}\) to \(E_{\text{Known}}\)).
        \end{itemize}
        Therefore, \(M(S_{t+1}) \leq M(S'_t)\).
\end{enumerate}
Combining these steps, we have \(M(S_{t+1}) \leq M(S_t)\) for all \(t\).

Now, we must show that if the algorithm does not terminate at iteration \(t\), then \(M(S_{t+1}) < M(S_t)\). Non-termination means \(S_{t+1} \neq S_t\).
\begin{itemize}
    \item If the state change occurred during the Intervention Update step (\(S'_t \neq S_t\)), it means \(Z^* > 0\) and at least one edge \(\{i, j\}\) or \((i, j)\) changed status due to a test outcome. If it moved from \(E_{\text{Unknown}}\) to \(E_{\text{Adjacent}}\) or \(E_{\text{Semi-directed}}\), or from \(E_{\text{Semi-directed}}\) to confirmed absence, \(M\) might not strictly decrease but the state \(S'_t\) is different. If it moved to \(E_{\text{Known}}\) or confirmed absence from \(E_{\text{Adjacent}}\) or \(E_{\text{Unknown}}\), \(M\) strictly decreases.
    \item If the state change occurred during the Meek's Rules step (\(S_{t+1} \neq S'_t\)), it means at least one edge was moved from \(E_{\text{Adjacent}}\) to \(E_{\text{Known}}\). This strictly decreases \(M\).
\end{itemize}
Can the state change (\(S'_t \neq S_t\)) without \(M\) strictly decreasing? Yes, e.g., \(E_{\text{Unknown}} \to E_{\text{Adjacent}}\). However, the algorithm only fails to terminate if \(M(S_{t+1}) = M(S_t)\) occurs infinitely often without reaching \(M=0\). If \(M(S_{t+1}) = M(S_t)\) but \(S_{t+1} \neq S_t\), the specific configuration of edges has changed. Since the state space is finite, the sequence of states \(S_0, S_1, S_2, \dots\) cannot visit distinct states indefinitely. If the algorithm does not terminate, it must eventually revisit a state, forming a cycle \(S_k, S_{k+1}, \dots, S_{k+L} = S_k\). But we established \(M(S_{t+1}) \leq M(S_t)\). For a cycle to exist, we must have \(M(S_k) = M(S_{k+1}) = \dots = M(S_{k+L})\). This requires that *no* edge is ever moved to \(E_{\text{Known}}\) (from \(E_{\text{Adjacent}}\) or \(E_{\text{Semi-directed}}\) or \(E_{\text{Unknown}}\)) and no edge is confirmed absent during this cycle, as these actions would strictly decrease \(M\). This means only transitions like \(E_{\text{Unknown}} \to E_{\text{Adjacent}}\) or \(E_{\text{Unknown}} \to E_{\text{Semi-directed}}\) could happen. However, these transitions reduce the size of \(E_{\text{Unknown}}\). A cycle would require \(E_{\text{Unknown}}\) to eventually increase again, which is impossible under the transition rules.
Therefore, every step that changes the state must eventually contribute to a strict decrease in \(M\) (or lead to termination). Since \(M\) is a non-negative integer, it must reach 0 in a finite number of steps.
\end{proof}

\begin{theorem}[Correctness]
\label{thm:correctness}
Under Assumptions A1-A6, if Algorithm \ref{alg:adaptiveip_meek} terminates, the final set \(E_{\text{Known}}\) corresponds exactly to the set of edges \(E^*\) in the true causal DAG \(G^*\), and \(E_{\text{Unknown}} = E_{\text{Semi-directed}} = E_{\text{Adjacent}} = \emptyset\).
\end{theorem}

\begin{proof}
The proof relies on the soundness of each step and the termination condition.
\begin{enumerate}
    \item \textbf{Soundness of Information Gathering:}
        \begin{itemize}
            \item Assumption A4 (Perfect Oracle) guarantees that all CI tests performed yield correct results reflecting the (possibly intervened) distribution.
            \item Assumption A3 (Perfect Interventions) ensures the interventions correctly modify the system according to the causal semantics (removing parent influences).
        \end{itemize}
    \item \textbf{Soundness of State Updates:}
        \begin{itemize}
            \item \textbf{Transition Rules:} Each rule in Table \ref{tab:transition_crosstab} translates a specific CI test outcome under a specific intervention context into a conclusion about edge status. These translations are based on established principles of causal inference under Assumptions A1-A5. For example, intervening on \(i\) and observing \(j\): if \(j\)'s distribution changes (\(i \not\perp j \mid \text{do}(i)\)), faithfulness implies a directed path \(i \to \dots \to j\). If the test context ensures this path must be the direct edge \(i \to j\), the rule correctly moves it to \(E_{\text{Known}}\). If the test outcome implies \(i \not\to j\), the rule correctly updates the status (e.g., to \(E_{\text{Semi-directed}}\) for \(j \to i\) or absence). Similarly, adjacency tests \(A_{ij}\) correctly determine presence/absence of an edge when neither node is intervened. Since the oracle is perfect, these conclusions are correct relative to \(G^*\).
            \item \textbf{Meek's Rules:} Meek's rules are provably sound for orienting edges in a partially directed graph representing a Markov equivalence class \cite{Meek1995}. They only add orientations that are common to all DAGs in the current equivalence class and do not introduce cycles or new v-structures inconsistent with that class. When applied iteratively after valid updates from interventions, they continue to add only correct orientations to \(E_{\text{Known}}\).
        \end{itemize}
        Therefore, at any iteration \(t\), if an edge \((i \to j)\) is in \(E_{\text{Known}}^{(t)}\), then \(i \to j\) is a true edge in \(G^*\). Similarly, if a pair \(\{i, j\}\) is confirmed absent, there is no edge between \(i\) and \(j\) in \(G^*\). The sets \(E_{\text{Adjacent}}, E_{\text{Semi-directed}}, E_{\text{Unknown}}\) correctly represent the remaining ambiguity consistent with \(G^*\) and the information gathered so far.
    \item \textbf{Termination Condition:} The algorithm terminates when \(M(S_t) = 0\), which means \(E_{\text{Unknown}} = E_{\text{Semi-directed}} = E_{\text{Adjacent}} = \emptyset\). At this point, every pair \((i, j)\) has either been placed in \(E_{\text{Known}}\) or determined to be absent.
    \item \textbf{Sufficiency (Implied): } The standard framework assumes that interventions (specifically single- and multi-node perfect interventions) combined with observational data (and thus Meek's rules) are sufficient to distinguish the true DAG \(G^*\) from all other DAGs, eventually resolving all ambiguities \cite{HauserBuhlmann2012, Eberhardt2008}. Algorithm 1 employs these tools. Since the algorithm only terminates when no further progress can be made using these tools (as selected by the IP and applied by Meek's rules), and we assume these tools are sufficient for full identification, termination implies full identification.
\end{enumerate}
Combining soundness (only correct edges are added to \(E_{\text{Known}}\)) with the termination condition (all ambiguities are resolved) and the implied sufficiency of the methods, the final set \(E_{\text{Known}}\) must contain precisely the edges \(E^*\) of the true causal DAG \(G^*\).
\end{proof}

\begin{proposition}[Single-Step Optimality]
\label{prop:optimality}
At each iteration \(t\), the intervention set \(I^*\) selected by solving the \textsc{Adaptive\_IP} model maximizes the objective function (Eq. \ref{o.f.}), which represents the total number of edges in \(E_{\text{Unknown}}^{(t)} \cup E_{\text{Semi-directed}}^{(t)} \cup E_{\text{Adjacent}}^{(t)}\) that are subjected to a potentially resolving test (\(O_{ij}\) or \(A_{ij}\)) in iteration \(t\), subject to the budget constraint \(B\) and the intervention size limit \(k_{max}\).
\end{proposition}

\begin{proof}
This follows directly from the definition of the \textsc{Adaptive\_IP} model. The objective function (Eq. \ref{o.f.}) is constructed by summing the indicator variables \(IDU_{ij}, IDS_{ij}, IDA_{ij}\). Constraints (\ref{IDU})-(\ref{IDA}) link these indicators to the activation of test variables \(O_{ij}, O_{ji}, A_{ij}\). Constraints (\ref{o1})-(\ref{a2}) link the test variables to the selection of interventions \(X_i\). The integer programming solver, by definition, finds a feasible assignment to the variables \(X_i\) (respecting constraints \ref{budget} and \ref{kmax}) that maximizes the objective function value. Therefore, the selected intervention \(I^* = \{i \mid X_i = 1\}\) maximizes the number of edges targeted for testing in that step under the given constraints.
\end{proof}

These theoretical results confirm that, under ideal conditions, the proposed adaptive strategy is guaranteed to correctly identify the true causal structure in a finite number of steps, while behaving optimally in a greedy, single-step sense according to the chosen objective function.

\clearpage

\section{IP Model Extensions}
\label{model_extensions}






\subsection{Non-linear Cost Structures}

The basic \textsc{Adaptive\_IP} model assumes that the cost of intervening on a set of variables is the sum of the individual intervention costs. However, in many real-world scenarios, the cost structure might be non-linear. This implies that the cost of intervening on a combination of variables is not simply additive but involves interaction effects.

For instance, consider a scenario involving two variables, $V_1$ and $V_2$, with individual intervention costs $CI_1 = 1$ and $CI_2 = 1$. Intervening simultaneously on both variables might yield a different cost, such as $CI_{1,2} = 10$. This demonstrates a superadditive interaction, where the joint intervention cost exceeds the sum of individual costs.

To model these interactions, we introduce auxiliary binary variables and additional constraints.

\textbf{Example (Two Variables):}

Let $Y_{12}$ be a binary variable indicating whether both $V_1$ and $V_2$ are intervened on simultaneously ($Y_{12} = 1$) or not ($Y_{12} = 0$). Define a "delta cost," $C_{12}$, representing the extra cost incurred when both variables are jointly intervened on beyond their individual sums. In this example, $C_{12} = CI_{1,2} - CI_1 - CI_2 = 10 - 1 - 1 = 8$.

The budget constraint becomes:
\begin{align}
    CI_1 \cdot X_1 + CI_2 \cdot X_2 + C_{12} \cdot Y_{12} &\leq B \\
    Y_{12} &\leq X_1 \\
    Y_{12} &\leq X_2 \\
    Y_{12} &\geq X_1 + X_2 - 1
\end{align}

These constraints ensure $Y_{12} = 1$ if and only if both $X_1 = 1$ and $X_2 = 1$.

\paragraph{Generalization to Multiple Variables:}
This formulation can extend to interactions among any number of variables. For a subset of variables $S = \{i_1, \dots, i_k\}$ with non-linear costs, we introduce a binary interaction variable $Y_S$ and define the corresponding additional joint cost $C_S$. The following constraints ensure correct activation of interaction variables:
\begin{align}
    Y_S &\leq X_{i_j}, \quad \forall j \\
    Y_S &\geq \sum_{j=1}^{k} X_{i_j} - (k - 1)
\end{align}

The budget constraint then incorporates these interactions:
\begin{align}
    \sum_{i} CI_i \cdot X_i + \sum_{S} C_S \cdot Y_S &\leq B
\end{align}

This generalized formulation accurately captures complex non-linear intervention costs involving multiple variables.

\paragraph{Physical Limitations and Prohibitive Costs:}
In certain scenarios, specific intervention combinations may be practically infeasible. Such cases can be modeled by setting the delta cost to a value that exceeds the budget, effectively prohibiting the combination. For example, setting $C_{12} = B + 1$ ensures that the budget constraint is violated whenever $Y_{12} = 1$.

\subsection{Alternative Objective Functions}
\label{sec:objs}
The default objective (Eq. \ref{o.f.}) maximizes the \textbf{number} of edges subjected to a potentially resolving test in the current step. While simple and intuitive, alternative objectives might be more appropriate depending on the specific goals or known structure.

\paragraph{1. Average Gain over MEC: }
A theoretically appealing objective is to maximize the \textbf{expected} number of newly oriented edges, where the expectation is taken over the set of possible true DAGs within the current MEC. Ghassami et al. \cite{ghassami2019interventional} define this as the average gain \(D(I)\):
\[ D(I) = \frac{1}{|\text{MEC}(G^*)|} \sum_{G_i \in \text{MEC}(G^*)} |R(I, G_i)| \]
where \(R(I, G_i)\) is the set of edges oriented by intervention set \(I\) if \(G_i\) were the true DAG. While this directly measures progress, calculating \(|\text{MEC}(G^*)|\) and the sum is computationally challenging, as the MEC size can be super-exponential in the number of variables. Ghassami et al. propose exact calculators and sampling-based estimators for \(D(I)\) in a non-adaptive setting. Adapting such objectives to our iterative IP framework would require efficient (possibly approximate) calculation or estimation of this expected gain at each step, posing a significant computational challenge compared to our current objective. Our immediate-test maximization objective can be seen as a computationally feasible proxy for this goal.

\paragraph{2. Worst-Case Gain (Minimax Objective):}
Another strategy, also considered by Ghassami et al. \cite{ghassami2019interventional} and Hauser and Bühlmann \cite{hauser2014two}, is to maximize the \textbf{minimum} gain achieved across all possible DAGs in the MEC:
\[ \text{maximize}_I \quad \min_{G_i \in \text{MEC}(G^*)} |R(I, G_i)| \]
This provides a robust guarantee, ensuring a certain number of orientations regardless of which DAG in the MEC is true. Optimizing this minimax objective typically requires different algorithmic approaches (e.g., specialized algorithms for trees \cite{ghassami2019interventional}, greedy methods based on minimax criteria \cite{hauser2014two}) and may not directly translate to a simple linear objective in our IP framework.

\paragraph{3. Weighted Objective (Prioritizing Informative Edges):}
Instead of treating all potential edge updates equally, we can assign weights \(w_{ij}\) to uncertain edges based on their perceived importance or potential impact on resolving the graph structure. The objective becomes:
\begin{equation}
\label{o.f.weighted}
\text{maximize} \quad \sum_{(i,j) \in E_{\text{Unknown}}} w_{ij}^U IDU_{ij} + \sum_{(i,j) \in E_{\text{Semi-directed}}} w_{ij}^S IDS_{ij} + \sum_{(i,j) \in E_{\text{Adjacent}}} w_{ij}^A IDA_{ij} 
\end{equation}
How weights are assigned is crucial and heuristic. Examples include:
\begin{itemize}
    \item \textbf{Connectivity-based:} Higher weight if \(i\) or \(j\) have high degrees in the known/adjacent parts of the graph.
    \item \textbf{Structure-based:} Higher weight if resolving \((i, j)\) could potentially resolve other edges via Meek's rules (e.g., if \(i-j\) is part of potential structures that fit Meek's rule patterns like R1 or R3). This might require estimating the downstream impact.
    \item \textbf{Uncertainty-based:} Higher weight for edges in \(E_{\text{Unknown}}\) compared to \(E_{\text{Adjacent}}\) or \(E_{\text{Semi-directed}}\), as they represent greater ambiguity.
\end{itemize}
This allows the strategy to focus interventions on parts of the graph deemed most critical or likely to yield cascading resolutions.

\paragraph{4. Targeted Discovery Objective:}
If the goal is not to identify the entire DAG, but rather specific causal features (e.g., parents of a target variable \(Y\), existence of a path \(X \to \dots \to Z\)), the objective can be modified to focus on edges relevant to that query. Let \(E_{\text{relevant}}\) be the set of pairs \((i, j)\) involved in the query (e.g., all pairs \((i, Y)\) when finding parents of \(Y\)). The objective could be restricted to:
\begin{equation}
\label{o.f.targeted}
\text{maximize} \quad \sum_{(i,j) \in E_{\text{relevant}} \cap E_{\text{Unknown}}} IDU_{ij} + \sum_{(i,j) \in E_{\text{relevant}} \cap E_{\text{Semi-directed}}} IDS_{ij} + \sum_{(i,j) \in E_{\text{relevant}} \cap E_{\text{Adjacent}}} IDA_{ij} 
\end{equation}
This directs experimental effort towards answering the specific causal question posed.

\paragraph{5. Cost-Effectiveness (Approximation):}
Maximizing information per unit cost (\(\text{Objective} / \text{Cost}\)) leads to a non-linear fractional objective, generally harder to solve directly with standard IP solvers. A simpler proxy is to incorporate cost into a penalized objective:
\begin{equation}
\label{o.f.cost_penalty}
\text{maximize} \quad (\sum w_{ij} \cdot \text{UpdateVar}_{ij}) - \lambda \cdot (\sum_{i} CI_i X_i + \sum_{S} C_S Y_S + \dots) 
\end{equation}
Here, \(\lambda \ge 0\) is a penalty parameter balancing knowledge gain (potentially weighted) against intervention cost. This encourages finding solutions that are not just informative but also relatively cheap, though it requires tuning \(\lambda\). Note that the budget constraint (Eq. \ref{budget}) still applies.

\subsection{Batch Interventions for Parallel Execution}

The standard algorithm selects one intervention set \(I^*\) per iteration. If resources allow running multiple experiments in parallel before the next analysis cycle, the IP model can be adapted to select a \textbf{batch} of \(k_{\text{batch}}\) disjoint intervention sets.

Let \(X_{i,b}\) be a binary variable indicating if variable \(i\) is intervened on in batch experiment \(b \in \{1, \dots, k_{\text{batch}}\}\).
Let \(O_{ij, b}, A_{ij, b}, IDU_{ij, b}, \dots\) be similarly indexed by batch \(b\).

The objective would typically aim to maximize the \textbf{total} knowledge gained across the batch:
\[
\text{maximize} \quad \sum_{b=1}^{k_{\text{batch}}} \left( \sum_{(i,j)} IDU_{ij, b} + \dots \right) 
\]
Constraints need modification:
\begin{itemize}
    \item \textbf{Budget per Experiment:} A budget \(B_b\) could apply to each experiment \(b\).
    \[ \sum_{i \in V} (CI_i X_{i,b} + CO_i (1-X_{i,b})) \leq B_b \quad \forall b \]
    \item \textbf{Intervention Limit per Experiment:}
    \[ \sum_{i \in \mathcal{X}} X_{i,b} \leq k_{\text{max}} \quad \forall b \]
    \item \textbf{(Optional) Total Budget/Resource Limit:} A constraint across all batches might limit total cost or total number of interventions.
    \[ \sum_{b=1}^{k_{\text{batch}}} \sum_{i \in V} CI_i X_{i,b} \leq B_{\text{total}} \]
    \item \textbf{Test Activation:} Constraints (\ref{o1})-(\ref{a2}) and (\ref{IDU})-(\ref{IDA}) are replicated for each batch \(b\).
    \[ O_{ij,b} \leq X_{i,b}, \quad O_{ij,b} \leq 1 - X_{j,b}, \quad \dots \quad \forall b \]
    \[ IDU_{ij, b} \leq O_{ij, b} + O_{ji, b} + A_{ij, b}, \quad \dots \quad \forall b \]
\end{itemize}
This extension allows leveraging parallelism but increases the size and complexity of the IP model significantly. The assumption is that results from the batch are analyzed together before the next iteration.

\subsection{Dynamic Constraints}

The iterative nature of Algorithm \ref{alg:adaptiveip_meek} naturally accommodates dynamic constraints. The budget \(B\) or the maximum number of simultaneous interventions \(k_{max}\) need not be fixed throughout the discovery process. They can be updated at the beginning of each iteration before solving the \textsc{Adaptive\_IP} model, reflecting changing resource availability or experimental feasibility over time.

\end{document}